\documentclass[twoside]{article}

\usepackage[accepted]{aistats2023}
\usepackage[round]{natbib}
\usepackage{xcolor}
\usepackage{booktabs}       
\usepackage{amsfonts}       
\usepackage{nicefrac}       
\usepackage{microtype}      
\usepackage{xcolor}         
\usepackage{graphicx}
\usepackage{caption}
\usepackage{subcaption, enumitem, xcolor, listings, amsmath}
\usepackage{stmaryrd}
\usepackage{amsthm}
\usepackage{pgfplots}
\usetikzlibrary{arrows}
\usepgfplotslibrary{groupplots, external}
\usepackage{wrapfig}

\usepackage{algorithm}
\usepackage{algorithmic}

\usepackage{amssymb}
\usepackage{url} 

\pgfplotsset{compat=1.17}

\newcommand{\dice}[0]{\texttt{Dice}}
\newcommand{\flip}[0]{\texttt{flip}}

\newcommand{\dbracket}[1]{{\llbracket #1 \rrbracket}}
\newcommand{\hoist}[0]{\mathtt{hoist}}
\newcommand{\ancestor}[0]{\mathtt{ancestor}}

\newcommand{\fh}[0]{\flip{}-hoisting}
\newcommand{\pathcond}[0]{\textsc{pc}}
\newcommand{\prog}[0]{\mathtt{p}}

\newtheorem{theorem}{Theorem}
\newtheorem{proposition}[theorem]{Proposition}

\newtheorem{definition}{Definition}

\definecolor{blue}{RGB}{158, 221, 255}

\begin{document}

\runningtitle{\texttt{flip}-hoisting: Exploiting Repeated Parameters in Discrete Probabilistic Programs}

%

%

\twocolumn[

\aistatstitle{\texttt{flip}-hoisting: Exploiting Repeated Parameters\\in Discrete Probabilistic Programs}

\aistatsauthor{ Ellie Y. Cheng \And Todd Millstein \And Guy Van den Broeck \And Steven Holtzen }

\aistatsaddress{ MIT \And UCLA \And UCLA \And Northeastern University } 
]

\begin{abstract}
  Many of today's probabilistic programming languages (PPLs) have brittle inference
  performance: the
  performance of the underlying inference algorithm is very sensitive to the
  precise way in which the probabilistic program is written. A standard way of 
  addressing this challenge in traditional programming languages is via
  \emph{program optimizations}, which seek to unburden the programmer from
  writing low-level performant code, freeing them to work at a higher-level of
  abstraction. 
  The arsenal of applicable program optimizations for PPLs to choose from is
  scarce in comparison to traditional programs; few of
  today's PPLs offer significant forms of automated program optimization. 
  In this work we develop a new family of program optimizations specific to
  discrete-valued knowledge compilation based PPLs. We identify a particular 
  form of program structure unique to these PPLs that 
  tangibly affects exact inference performance in these programs: 
  \emph{redundant random variables} -- variables with repeated parameters and
  inconsistent path conditions.
  We develop a new program analysis and associated optimization called \fh{} that identifies 
  these redundancies and optimizes them into a single random variable. 
  We show that \fh{} yields inference speedups of up to 60\% on 
  applications of probabilistic programs such as Bayesian networks and probabilistic verification.
\end{abstract}

\section{INTRODUCTION}
Probabilistic programming languages (PPLs) have immense promise as a
general-purpose and widely accessible probabilistic modeling solution, but
they suffer from brittle inference performance: two programs with the same semantics,
written with subtly different syntax, can have drastically different inference
performance. As a consequence, many of today's PPLs are difficult to use for
non-experts that cannot navigate the subtle ways in which the structure of the
program impacts inference. For example, \citet{gorinova2020automatic} showed 
that subtle re-parameterizations of programs, such as centering, yields 
radically different inference performance for Hamiltonian Monte Carlo inference; 
an average end-user should ideally be blissfully unaware of this.

The standard way in which the challenge of brittle performance is addressed in 
traditional programming languages is via \emph{program
optimizations}~\citep{aho1986compilers}. Optimizations permit the programmer to
work at a high-level of abstraction, and delegate to the compiler the job of
synthesizing a more efficiently written program; this separation of concerns is
critical to the design of modern compilers and programming languages such as
LLVM~\citep{lattner2008llvm}. However, in contrast to traditional programming languages,
there is a much broader variety of optimizations for PPLs due to the large number of
possible inference algorithms: the effect an optimization might have on a
program is intimately connected with which inference algorithm is ultimately
used. For instance, \citet{gorinova2020automatic} identified an optimization
that is specific to HMC, and
\citet{huang2017compiling} identified optimizations specific to parallel Gibbs sampling. 
An optimization that may work for a PPL may be ineffective in another. 
Hence, optimizing general-purpose PPLs necessitates developing a broad suite of 
optimizations specific to many kinds of inference algorithms.

One particular family of inference algorithms in need of new optimizations 
is those of \emph{discrete} PPLs. Discrete PPLs such as
\dice{} and \texttt{ProbLog} enable new forms of inference since
they only apply to programs with discrete random
variables~\citep{holtzen2020scaling,fierens2012inference}.  The dominant form of
inference for these kinds of PPLs is \emph{exact inference via knowledge
compilation}~\citep{darwiche2002knowledge,sang2005performing,chavira2008probabilistic,
niepert2015learning, mateescu2008andor, chavira2006compiling, sanner2005affine}.
The core idea of the approach is to compile the probabilistic program 
into a representation that affords fast exact inference.  

The performance of exact inference via knowledge compilation is sensitive to the
way in which programs are written in ways that can be surprising or
unintuitive for novice programmers. We identify one particularly pernicious form
of program structure that is easy for a novice programmer to accidentally
introduce and has a negative impact on inference performance: \emph{redundant
random variables}. Two syntactic random variables are \emph{redundant} if they
have the same parameters and have inconsistent path conditions.
Redundancies do not affect the performance of sampling-based inference algorithms:
a single path through the program can never encounter both redundant random
variables. However, redundancies can significantly impact knowledge compilation 
based inference methods like \dice{} because it performs exact inference, and
hence implicitly represents the joint configuration of all random variables.
In this work, we propose a new family of probabilistic program optimizations called
\fh{} that unifies redundant random variables into a single random variable through
program analysis and transformation.
This optimization is specific to the knowledge compilation approach utilized by 
discrete PPLs such as \dice{}.
The key behind our approach is a sound but incomplete
\emph{branch-sensitive analysis} of probabilistic programs that statically
determines when it is safe to re-use a random variable. This analysis is
efficient in the size of the program and is effective on realistic examples.

In Section~\ref{sec:motiv} we give a motivating example that provides
intuition behind how \fh{} works. 
Section~\ref{sec:fliphoisting} formally describes the \fh{}
procedure and argues for its correctness. 
Section~\ref{sec:experiments} 
describes how to implement \fh{} for \dice{} and
demonstrates empirically that \fh{} 
is an effective optimization by showing (1) that practical problems 
that exhibit opportunities for \fh{} exist in the wild on examples from 
probabilistic graphical models and probabilistic verification; and (2)
that our automated \fh{} approach
reduces the runtime of \dice{} inference 
on these examples, giving new
state-of-the-art performance for probabilistic programming languages on these
tasks. Section~\ref{sec:related_work} gives an overview of related work.
Section~\ref{sec:conclusion} concludes.

\subsection{Motivating Example and Overview}
\label{sec:motiv}
\begin{figure}
    \begin{subfigure}[t]{\linewidth}
        \begin{center}
        \begin{tabular}{c}
        \begin{lstlisting}[language=caml,escapechar=|, mathescape=true,
        numbers=left,
        stepnumber=1,
        basicstyle=\ttfamily\scriptsize]
let x = flip$_1$ 0.1 in
let z = flip$_2$ 0.2 in
let y = if x && z then|\label{line:if}|
    flip$_3$ 0.3
  else if x && !z then
    flip$_4$ 0.2
  else flip$_5$ 0.3
in y
        \end{lstlisting}
        \end{tabular}
        \end{center}
        \caption{No optimizations.}
        \label{fig:simple-ex}
    \end{subfigure}
    \qquad
    
    \begin{minipage}{\linewidth}
        \centering
        \begin{subfigure}[t]{0.4\linewidth}
        \centering
            \begin{lstlisting}[language=caml,escapechar=|, mathescape=true,
            numbers=left,
            stepnumber=1,
            basicstyle=\ttfamily\scriptsize]
let x = flip$_1$ 0.1 in 
let z = flip$_2$ 0.2 in
let tmp = flip$_{3,5}$ 0.3 in
let y = if x && z then 
    tmp
  else if x && !z then 
    flip$_4$ 0.2
  else tmp
in y
            \end{lstlisting}
            \caption{Valid \flip{}-hoisting. }
            \label{fig:simple-ex-hoisted}
        \end{subfigure}
        \qquad  \quad
        \begin{subfigure}[t]{0.4\linewidth}
            \begin{lstlisting}[language=caml,escapechar=|, mathescape=true,
            numbers=left,
            stepnumber=1,
            basicstyle=\ttfamily\scriptsize]
let x = flip$_1$ 0.1 in 
let tmp = flip$_{2,4}$ 0.2 in
let z = tmp in
let y = if x && z then 
    flip$_3$ 0.3 
  else if x && !z then 
    tmp
  else flip$_5$ 0.3
in y
            \end{lstlisting}
            \caption{Invalid \texttt{flip}-hoisting.}
            \label{fig:invalid-ex}
        \end{subfigure}
    \end{minipage}

\caption{Examples of \fh{} on \dice{} programs. Subscripts on \flip{}s are only used to uniquely refer to them in our discussion.}
\label{fig:fig1}
\end{figure}

Throughout this paper, examples will be written in \dice{}~\citep{holtzen2020scaling}. 
\dice{} is a functional probabilistic programming language that
supports discrete random variables; in \dice{}, the syntax \flip{}~$\theta$
denotes a Boolean random variable that is true with probability~$\theta$. We introduce subscripts to \flip{}s in the examples as auxiliary notation to uniquely refer to them in our discussion.

Consider the example probabilistic program in Figure~\ref{fig:simple-ex}. 
The program encodes a simple distribution with a Boolean random variable \texttt{y} that
is conditioned on the values of \texttt{x} and \texttt{z}, which are themselves
Boolean random variables with probability $0.1$ and $0.2$ respectively.
The program has some potentially redundant \flip{}s; \texttt{flip 0.3} occurs twice in the program as \flip{}$_3$ and \flip{}$_5$.
So, we ask: can we optimize this program to a
version where these \texttt{flip 0.3} occurs exactly once? Intuitively, instead of flipping two coins, we would like
to flip a single coin instead, without changing the semantics of the program.
We can indeed do so by \emph{hoisting}, wherein we introduce a variable \texttt{tmp} in order
to use a single \flip{} in place of the original two occurrences. In Figure~\ref{fig:simple-ex-hoisted}, we show the optimized version of the program after hoisting
\flip{}$_3$ and \flip{}$_5$.

\begin{figure}[t]
    \begin{subfigure}[t]{0.45\columnwidth}
        \centering
        \begin{tikzpicture}[x=0.5cm,y=0.9cm]
            \tikzset{vertex/.style = {shape=circle,draw,minimum size=0.8em, font=\small}}
            \tikzset{square/.style = {shape=rectangle, draw, font=\small}}
            \tikzset{tedge/.style = {-}}
            \tikzset{fedge/.style = {dash pattern=on 2pt off 1pt}}
            
            \node[vertex] (f1) at  (0,0) {f$_1$};
            \node[vertex] (f2) at  (-1,-1) {f$_2$};
            \node[vertex, fill=blue] (f3) at  (-2,-2) {f$_3$};
            \node[vertex] (f4) at  (0,-2) {f$_4$};
            \node[vertex, fill=blue] (f5) at  (2,-2) {f$_5$};
            \node[square] (t) at (-1,-3) {T};
            \node[square] (f) at (1,-3) {F};
            \draw[tedge] (f1) to (f2);
            \draw[tedge] (f2) to (f3);
            \draw[fedge] (f2) to (f4);
            \draw[fedge] (f1) to (f5);
            \draw[tedge] (f3) to (t);
            \draw[fedge] (f3) to (f);
            \draw[tedge] (f4) to (t);
            \draw[fedge] (f4) to (f);
            \draw[tedge] (f5) to (t);
            \draw[fedge] (f5) to (f);
        \end{tikzpicture}
        \caption{BDD of Figure~\ref{fig:simple-ex}, the original program.}
        \label{fig:simple-ex-bdd}
    \end{subfigure}
    ~
    \quad
    \begin{subfigure}[t]{0.45\columnwidth}
        \centering
        \begin{tikzpicture}[x=0.5cm,y=0.9cm]
            \tikzset{vertex/.style = {shape=circle,draw,minimum size=1em, font=\small}}
            \tikzset{square/.style = {shape=rectangle, draw, font=\tiny}}
            \tikzset{tedge/.style = {-}}
            \tikzset{fedge/.style = {dash pattern=on 2pt off 1pt}}
            
            \node[vertex] (f1) at  (0,0) {f$_1$};
            \node[vertex] (f2) at  (-1,-1) {f$_2$};
            \node[vertex, scale=0.8, fill=blue] (f35) at  (2,-2) {f$_{3,5}$};
            \node[vertex] (f4) at  (0,-2) {f$_4$};
            \node[square] (t) at (-1,-3) {T};
            \node[square] (f) at (1,-3) {F};
            \draw[tedge] (f1) to (f2);
            \draw[tedge] (f2) to (f35);
            \draw[fedge] (f2) to (f4);
            \draw[fedge] (f1) to (f35);
            \draw[tedge] (f35) to (t);
            \draw[fedge] (f35) to (f);
            \draw[tedge] (f4) to (t);
            \draw[fedge] (f4) to (f);
        \end{tikzpicture}
        \caption{BDD of Figure~\ref{fig:simple-ex-hoisted}, the program with valid \flip{}-hoisting.}
        \label{fig:simple-ex-hoisted-bdd}
    \end{subfigure}
\caption{BDDs of example programs in Figure~\ref{fig:fig1}.}
\label{fig:fig1-bdd}
\end{figure}
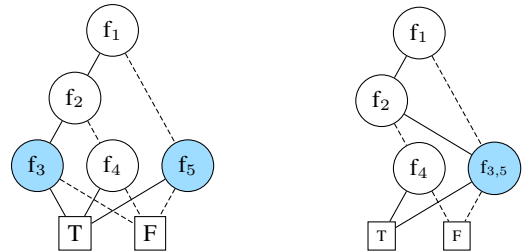

Removing redundant \texttt{flip}s is an optimization that specifically improves the 
performance of knowledge compilation, similar to how centering is an optimization unique to improving HMC~\citep{gorinova2020automatic}. To see this, we will inspect some details of how 
\dice{} inference works. \dice{} compiles 
probabilistic programs into \textit{binary decision diagrams} (BDDs) and computes
queries through \emph{weighted model counting}~\citep{holtzen2020scaling}. 
Each of the nodes in a BDD
corresponds to a \flip{} in the program; the solid edge leaving a node has a weight equal to the
parameter $\theta$ of the \flip{} and the dotted edge leaving a node has a weight equal to 
$1 - \theta$. Performing weighted model counting on a BDD then includes tracing through each path
from the leaf nodes to the root.
Figure~\ref{fig:simple-ex-bdd} shows the BDD of Figure~\ref{fig:simple-ex} with each node labeled with the \flip{} subscripts. 
When a \flip{} is hoisted
the BDD consequently \emph{reduces in size} as redundant nodes are removed.
Combining \flip{}$_3$ and \flip{}$_5$ results in the node being shared in the corresponding BDD in Figure~\ref{fig:simple-ex-hoisted-bdd},
resulting in a smaller BDD. The size of the BDD is the critical 
component in determining the ultimate runtime of \dice{} inference, and 
so reducing size has practical improvements on runtime~\citep{holtzen2020scaling}.

We have shown that \texttt{flip 0.3} can be optimized to occur only once instead of twice. 
Can we do the same for \texttt{flip 0.2}, which also occurs twice in the example program?
The answer is \emph{no} -- sound hoisting is not a simple matter of unifying all
syntactically identical random variables. Consider the \emph{invalid} program in
Figure~\ref{fig:invalid-ex}, in which we hoist \flip{}$_2$ and \flip{}$_4$.
This hoisting of the two \texttt{flip}s changes the semantics of the
original program. In the original program it is possible that \texttt{x = true,
z = false}, and \texttt{y = true}; this assignment is not possible in the
hoisted version since \texttt{y} is constrained to be equal to \texttt{z} in
this case. This invalid hoisting introduced a \emph{spurious dependence} between
the two \texttt{flip}s, coupling them when the semantics of the original program
depended on them being decoupled. The spurious dependence induces incorrect inference
results as the semantics of the original program has been altered.

What is a sufficient condition for ensuring that a hoisting is valid? One is
that there is no path through the probabilistic program that encounters both
\texttt{flip}s; we call such \texttt{flip}s \emph{redundant}. This is
the case for \texttt{flip 0.3}. The key observation is that the
\texttt{if}-expression on Line~\ref{line:if} in the original program in Figure~\ref{fig:simple-ex} has mutually exclusive branches, guaranteeing
that the path conditions for the two separate instances of \texttt{flip 0.3}
are inconsistent; in this case, it is safe to flip
a single coin instead of two without changing the semantics of the program. 
We note that while our motivating example only has Bernoulli random variables,
our approach naturally applies to categorical distributions
as well, since they can be represented as a sequence of 
Bernoulli random variables~\citep{holtzen2020scaling,sang2005performing}.

In this paper we give the first analysis for identifying and merging redundant
\texttt{flip}s in probabilistic programs. In general, determining if it is safe
to merge two \texttt{flip}s is a special case of \emph{reachability analysis},
and so is undecidable for general programs. Hence, we identify a sound
but incomplete strategy that is effective for existing \dice{} programs~\citep{aho1986compilers}. 
For example, our approach is able to perform the hoisting
shown in Figure~\ref{fig:simple-ex-hoisted} automatically, and never performs the
invalid hoisting shown in Figure~\ref{fig:invalid-ex}.
We will show that \fh{} is efficient in the size of the program, prove that it
is sound, and show empirically that it improves the performance of existing probabilistic
programs derived from graphical models and verification.

\section{\texttt{flip}-HOISTING}
\label{sec:fliphoisting}

In this section we formally introduce \fh{}. In order to unambiguously refer to
\flip{}s in a program, we assume that each one has a unique
identifier; we assign these identifiers as subscripts as shown in
Figure~\ref{fig:fig1}. We denote syntactic probabilistic programs as $\prog$ and let $\dbracket{\prog}$ denote the probability distribution on the values
returned by $\prog$.

\emph{Hoisting}, denoted $\hoist{}(\prog, i, j)$, is the function that transforms a program $\prog$ by inserting a new \flip{} -- the hoisted \flip{} -- at an ancestor
node of \flip{}$_i$ and \flip{}$_j$ in the abstract syntax tree (AST) -- $\ancestor{}(\prog, i, j)$ -- and replacing \flip{}$_i$ and
\flip{}$_j$ with a reference to the hoisted \flip{}. $\hoist{}(\prog, i, j)$
assumes \flip{}$_i$ and \flip{}$_j$ have the same parameter.
\footnote{In general there are multiple possible places where the new 
\flip{} can be inserted, but that is irrelevant here.} 

For instance, if $\prog_{ex}$
is the program in Figure~\ref{fig:simple-ex}, then $\hoist{}(\prog_{ex}, 3, 5)$
outputs the program in Figure~\ref{fig:simple-ex-hoisted}. Hoisting itself is
efficient and implementable as a single pass over the syntax of the program; the
challenge is knowing when hoisting is \emph{sound}:

\begin{definition}[Sound hoisting]
  For a probabilistic program $\prog$, hoisting \flip{}s $i$ and $j$ is sound
if $\dbracket{\prog} = \dbracket{\hoist{}(\prog, i, j)}$.
\end{definition}

This definition does not yield an algorithm because determining whether a
hoisting is sound is hard: it requires reasoning about whether or not two
probabilistic programs are equivalent, which is a challenging computational
task. Hence we require an assumption that aids in implementation. One route
is analyzing the \emph{path conditions} for each \flip{}, a familiar concept
from symbolic execution~\citep{king1976symbolic}:

\begin{definition}[Path conditions]
  The \emph{path conditions} $\pathcond{}(\prog, i)$ for \flip{} $i$ in probabilistic
  program $\prog$ is the set of necessary and sufficient conditions on \flip{}s that
  ensure execution reaches \flip{} $i$.
\end{definition}
For instance, $\pathcond(\prog_{ex}, 3) = \flip{}_1~\land~
\flip_2$, since if \flip{}$_3$ is executed then the condition of the
\texttt{if}-statement on Line~\ref{line:if} of Figure~\ref{fig:simple-ex} must be true, which implies that
these two \flip{}s are true. The path conditions yield a more tractable test
for \flip{} redundancy:

\begin{theorem}[Path redundancy]
  For program $\prog$, hoisting \flip{}s $i$ and $j$ is sound if the \flip{}s have the same
  parameter value and
  $\pathcond(\prog, i)$ is inconsistent with $\pathcond(\prog, j)$.
  \label{thm:pathred}
\end{theorem}
\begin{proof}[Proof sketch]
  A \emph{trace} $\tau$ through a program $\prog$ is a total assignment to \flip{}s
  encountered during an execution of the program; for example, one trace through
  Figure~\ref{fig:simple-ex} is $\tau = \{\flip{}_1 = true, \flip{}_2 = false,
  \flip{}_4 = true\}$. This information uniquely determines the result of the
  execution along this trace and the probability of the trace.
  
  There is a bijection between the traces of $\prog$ and the
  traces of $\hoist{}(\prog, i, j)$ that ensures that the programs are
  equivalent. In other words, a tuple of traces is in the bijection if both the resulting values and probabilities of the traces are equivalent. Thus, the bijection implies that the two programs have the same resulting probability distributions. 

  The most interesting traces are those that encounter one of the hoisted
  $\flip{}$s --- we will give a bijection between the traces of $\prog$ and
  $\hoist{}(\prog, i, j)$ for this case. 
  Consider a trace $\tau$ through $\prog$ that includes $\flip{}_i$.  Since
  $\pathcond(\prog, i)$ is inconsistent with $\pathcond(\prog, j)$, it must 
  be the case that $\tau$
  does not include $\flip{}_j$. 
  Now let $\prog_h = \hoist{}(\prog, i, j)$, and assume that $\flip{}_i$ is
  hoisted to and relabeled to a new variable $\flip{}'$ in the hoisted program.
  By the definition of the $\hoist{}$
  function there exists a trace $\tau'$ through $\hoist{}(\prog, i, j)$ that is
  identical to $\tau$
  but with $\flip{}_i$ replaced by $\flip{}'$
  with the same value in the trace.  Similarly, for each trace through
  $\hoist{}(\prog, i, j)$ that goes through the line of code where $\flip{}_i$
  was in $\prog$, there exists an equivalent trace in $\prog$ containing
  $\flip{}_i$ in place of $\flip{}'$.
\end{proof}

Path redundancy reduces the problem of checking hoisting soundness to
satisfiability, which is still too computationally hard to be implemented as a
practical optimization.
Next, we consider two strengthenings of path redundancy that capture common
cases occurring in probabilistic programs; these strengthenings 
yield efficient soundness checks.

\subsection{Local Hoisting}
One of the strictest strengthenings of path redundancy is \emph{local
redundancy}, which avoids the need to reason about the intricacies of
\texttt{if}-statement conditions altogether.

\begin{definition}[Locally redundant \flip{}s]
  Two \flip{}s $i$ and $j$ are \emph{locally redundant} if they have the same
parameter value and appear in disjoint branches of the same
if-statement.
\end{definition}

Determining whether or not two \flip{}s are locally redundant is an efficient
syntactic check on the program. Figure~\ref{fig:simple-ex} gives an example
where \flip{}$_3$ and \flip{}$_5$ are locally redundant: they occur
in disjoint branches of the \texttt{if}-expression on Line~\ref{line:if}.

\begin{proposition}
  It is sound to hoist locally redundant \flip{}s.
\end{proposition}
The proof follows from Theorem~\ref{thm:pathred} and the fact that two locally
redundant \flip{}s by definition have inconsistent path conditions.

While seemingly simple, local hoisting is already a surprisingly powerful and
general optimization. 
Further,
local hoisting applies to a surprisingly broad set of programs including those
encoding Bayesian networks as well as existing programs from probabilistic verification.

\subsection{Global Hoisting}
\label{s:global-hoisting}

\begin{figure}
\centering
\begin{minipage}{0.4\columnwidth}
\centering
\begin{subfigure}[t]{\columnwidth}
  \centering
  \begin{lstlisting}[language=caml, 
    numbers=left,
    stepnumber=1, basicstyle=\ttfamily\scriptsize, mathescape=true]
let x = flip$_1$ 0.1 in 
let y = if x then 
    flip$_2$ 0.2 
  else flip$_3$ 0.3 in
let z = if !x then 
    flip$_4$ 0.2 
  else flip$_5$ 0.4 in 
(y, z)
\end{lstlisting}
\caption{Example unhoisted program $\prog_g$.}
\label{fig:global-ex}
\end{subfigure}\\

\begin{subfigure}[t]{\columnwidth}
  \centering
  \begin{lstlisting}[language=caml, 
    numbers=left,
    stepnumber=1, basicstyle=\ttfamily\scriptsize, mathescape=true]
let x = flip$_1$ 0.1 in 
let tmp = flip$_{2, 4}$ 0.2 in
let y = if x then tmp 
  else flip$_3$ 0.3 in
let z = if !x then tmp 
  else flip$_5$ 0.4 in 
(y, z)
    \end{lstlisting}
    \caption{Global \flip{}-hoisting. }
    \label{fig:global-ex-hoisted}
  \end{subfigure}
\end{minipage}
\qquad \quad
\begin{minipage}{0.4\columnwidth}
\centering
\begin{subfigure}[b]{\columnwidth}
  \centering
  \begin{lstlisting}[language=caml,
    numbers=left,
    stepnumber=1, escapechar=|, basicstyle=\ttfamily\scriptsize, mathescape=true]
  
let x = flip$_1$ 0.1 in 
|\tiny\colorbox{orange}{{\{x=1}\}}|
let y = if x then 
  |\tiny\colorbox{cyan}{\{1=true\}}| flip$_2$ 0.2 |\label{line:constraint}|
  else 
  |\tiny\colorbox{cyan}{\{1=false\}}| flip$_3$ 0.3 
in
let z = if !x then 
  |\tiny\colorbox{cyan}{\{1=false\}}| flip$_4$ 0.2 
  else  
  |\tiny\colorbox{cyan}{\{1=true\}}| flip$_5$ 0.4 
in (y, z)
    \end{lstlisting}
    \centering
    \caption{The program $\prog_g$ annotated with data-flow facts.}
    \label{fig:global-ex-annot}
  \end{subfigure}
\end{minipage}
\caption{Example of global \fh{}.}
\label{fig:fig3}
\end{figure}

Local hoisting identifies hoisting opportunities within a single \texttt{if}-expression. In this section we develop a separate
analysis, which we call \emph{global hoisting}, that retains the tractability of local hoisting while finding hoisting
opportunities that span multiple \texttt{if}s. 

Consider the minimal example in Figure~\ref{fig:global-ex} that we label
$\prog_g$. In this program \flip{}$_2$ and \flip{}$_4$ are path redundant since
$\pathcond(\prog_g, 2) = \flip{}_1$ and $\pathcond(\prog_g, 4) = \neg\flip{}_1$,
which are clearly inconsistent logical sentences. However, these two
\flip{}s are \emph{not} locally redundant. We would like to efficiently certify
that it is safe to hoist these \flip{}s and others similar to them.
To accomplish this we will perform a form of \emph{data-flow analysis} on
the program~\citep{aho1986compilers}. The essence behind a data-flow analysis is to traverse the program
and collect a set of facts -- stated as logical propositions -- that hold at
each point in the program. By suitably constraining the structure of these facts
we ensure that the analysis is efficient.

Figure~\ref{fig:global-ex-annot} shows $\prog_g$ annotated with the data-flow
facts required to perform global hoisting. We track two kinds of facts: (1)
\emph{aliasing facts}, marked with orange boxes,
that relate local variables to the \flip{}s that they must be equal to; and (2)
\emph{constraint facts}, in cyan boxes, that list
assignments to \flip{}s that are implied by \texttt{if}-statement conditions.
For instance, at the point where \flip{}$_2$ occurs it must be the case
that \texttt{\{1 = true\}} because we know from the aliasing facts that
\texttt{x} is assigned to be equal to \flip{}$_1$, and we know that the
condition constrains \texttt{x} to be \texttt{true} since the branch was taken. The data-flow analysis records these aliasing facts and constraint facts in a linear pass through the AST.

Once the data-flow analysis is complete, it is straightforward to use the set of
constraint facts that hold at each \flip{} to determine when it is safe to
hoist. To check if \flip{}$_i$ and \flip{}$_j$ are redundant, we check if their corresponding
constraint facts are inconsistent, determined by if they disagree on any
assignments to literals, which is efficient. This is done with a simple iteration through the constraint facts of \flip{}s with equal parameters to determine if there is a variable assigned to \texttt{true} in one and to \texttt{false} in the other. By
definition, inconsistency of constraint facts implies path redundancy of the
two \flip{}s, and hence hoisting these two will be sound.

There are a few more details about the data-flow analysis that are necessary to
make it work. First, in order to efficiently construct the constraint facts, we
only derive such facts from the conditions of \texttt{if}-statements that are \emph{conjunctions of
variables} for the \texttt{then} branch and \textit{disjunctions of variables} for the \texttt{else} branch, which can be analyzed in a simple linear pass; for other more complex conditions, we take a best-effort approach, recording partial facts. Next, at \emph{join points} -- the points in the program at which two
branches merge back into a single flow of execution -- we take the
intersection of all data-flow facts to conservatively ensure
soundness.

\section{IMPLEMENTATION \& EVALUATION}
\label{sec:experiments}
In order for an optimization to be useful it must (1) be fast to implement in practice; (2)
have a potential to significantly improve performance; and (3) benefit existing
practical programs.
In this section we show
these three criteria hold for \fh{}.
First, Section~\ref{sec:implementation} shows how to implement \fh{} in an 
efficient way for the \dice{} system~\citep{holtzen2020scaling} -- 
this implementation requires a minor generalization to \fh{} that preserves an 
additional \emph{ordering invariant}. With this invariant we can prove 
that \fh{} can help and never hurt \dice{}'s inference performance, making it 
useful as a ``run-by-default'' optimization. Next Section~\ref{sec:synthetic} 
gives a best-case scenario on a synthetic set of programs and shows that 
\fh{} has the potential to speed up inference significantly. 
Finally, Section~\ref{sec:realistic} evaluates \fh{} on a 
broad family of existing probabilistic programs and shows that
\fh{} never hurts and sometimes significantly helps performance. It also shows
that there exist programs in the 
wild -- such as programs from the literature on probabilistic 
verification~\citep{HoltzenCAV21} -- that naturally
exhibit opportunities to hoist \flip{}s and benefit significantly from the optimization.~\footnote{All our code will be released 
as open-source. Our experiments were carried out on a Xeon E5-2640 CPU with 512GB RAM.}

\subsection{Implementation in \dice{}}
\label{sec:implementation}

To validate the effectiveness of this optimization in practice we implemented it
in the \dice{} probabilistic programming system~\citep{holtzen2020scaling}. 
\dice{} supports exact inference and
supports procedures, control-flow, categorical distributions on discrete values, and bounded
recursion.  \dice{} works by compiling a probabilistic program
into a \emph{binary decision diagram} (BDD), a tractable probabilistic model
that supports linear-time probabilistic
inference~\citep{darwiche2002knowledge,ProbCirc20}.

Unlike many other PPLs, the performance of \dice{} is 
extremely sensitive to the \emph{order in which variables are introduced in the
program}~\citep{holtzen2020scaling}. This is because the variable ordering of 
the BDD is dictated by this order, and the size of the compiled BDD -- which
is one of the primary factors that determines the runtime performance of
inference for the program -- is heavily influenced by the variable order. For
example, in Figure~\ref{fig:simple-ex}, the BDD variable corresponding to
\flip{}$_1$ will appear in the variable order before the BDD variable corresponding to \flip{}$_2$, as they are always executed in that order.

Ultimately our goal is to design an optimization that can never hurt -- and
often help -- inference performance. Hence, it is critical in the context of
\dice{} that \fh{} maintains the variable ordering of the program.
We therefore extend the \fh{}
transformation so that whenever it hoists a \flip{} it also hoists earlier
\flip{}s as necessary to maintain the original order.

\begin{definition}[Ordering-maintaining]
  If the ordering of the \flip{} definitions are the same between $\prog{}$ and $\hoist{}(\prog, i, j)$ -- where branches of \texttt{if}-expressions are considered interchangeable in ordering -- then $\hoist{}(\prog, i, j)$ is ordering-maintaining. 
\end{definition}

We implement ordering-maintaining \fh{} by also hoisting all the \flip{}s defined in between $\ancestor{}(\prog, i, j)$ and \flip$_i$ and between $\ancestor{}(\prog, i, j)$ and \flip$_j$ to be defined right before the hoisted \flip{}.
To minimize this extra work, we implement $\ancestor{}(\prog, i, j)$ to always return the lowest common ancestor (LCA). 
With order-maintaining 
\fh{}, we can state the following:

\begin{theorem}
  \label{thm:size}
  Let $|\prog|$ be the number of nodes in the BDD compiled from $\prog$ and let
$i$ and $j$ be \flip{} indexes in $\prog$. If hoisting preserves the program
variable order, then $|\hoist{}(\prog, i, j)| \le |\prog|$.
\end{theorem}
We will sketch a proof here to avoid going into the details of \dice{}
compilation.  Since the order of variables in the program is unchanged after
hoisting, so too is the order of variables in the compiled BDD.  Then, hoisting
can be thought of as relabeling \flip{}s $i$ and $j$ to have a common label. 
This does not change the size of the BDD. The resulting BDD will then be reduced
to canonical form, which may decrease the size; this is where we may profit from
\fh{}, as there may be more compression opportunities.

Hence, order-preserving hoisting can never hurt -- but can often help, as we
will see -- \dice{} compilation performance. One thing to note is that
it is not always possible to perform an order-preserving hoisting. In
particular, for global hoisting, there can be the situation that redundant \flip{}s
cannot be hoisted without breaking some ordering for the \flip{}s in between the
redundant \flip{}s. For example, suppose there is a \flip{}$_i$ and a \flip{}$_j$ with the same parameters and inconsistent path conditions. Let \flip{}$_i$ be defined strictly before \flip{}$_j$. If there is a \flip{}$_k$ defined in between them, \flip{}$_i$ and \flip{}$_j$ cannot be hoisted. \flip{}$_i$ would demand \flip{}$_k$ be inserted \textit{after} the newly hoisted variable, but \flip{}$_j$ needs \flip{}$_k$ to be inserted \textit{before} the newly hoisted variable, producing conflicting orderings. So, to satisfy the conditions of Theorem~\ref{thm:size}, it is occasionally necessary to forego hoisting opportunities to preserve order.

\begin{figure}
\centering
\begin{subfigure}[b]{0.8\columnwidth}
    \hspace{-1em}

\pgfplotstableset{row sep=crcr}
\begin{tikzpicture}

\definecolor{color0}{rgb}{0.274509803921569,0.509803921568627,0.705882352941177}
\definecolor{color1}{rgb}{0.603921568627451,0.803921568627451,0.196078431372549}

\begin{axis}[
height=100px,
legend cell align={left},
legend style={
  fill opacity=0.8,
  draw opacity=1,
  text opacity=1,
  at={(0.03,0.97)},
  anchor=north west,
  draw=white!80!black
},
tick align=outside,
tick pos=left,
x grid style={white!69.0196078431373!black},
xlabel={$n$},
xmajorgrids,
xmin=-0.65, xmax=13.65,
xtick style={color=black},
y grid style={white!69.0196078431373!black},
ylabel={Time (s)},
ymajorgrids,
ymin=-0.63864, ymax=14.02524,
ytick style={color=black},
width=\textwidth,
font=\small
]
\addplot [semithick, color0, mark=*, mark size=3, mark options={solid}]
table {%
0 0.0211000442504883\\
1 0.0273000001907349\\
2 0.0290999412536621\\
3 0.0319000482559204\\
4 0.0559999942779541\\
5 0.0564000606536865\\
6 0.0722000598907471\\
7 0.139299988746643\\
8 0.302000045776367\\
9 0.620500087738037\\
10 1.33790004253387\\
11 2.92149996757507\\
12 6.34210014343262\\
13 13.7208995819092\\
};
\addlegendentry{Original}
\addplot [semithick, color1, mark=triangle*, mark size=3, mark options={solid}]
table {%
0 0.0209000110626221\\
1 0.0211000442504883\\
2 0.0286999940872192\\
3 0.0295000076293945\\
4 0.0298999547958374\\
5 0.0322999954223633\\
6 0.0369999408721924\\
7 0.0398000478744507\\
8 0.076200008392334\\
9 0.125\\
10 0.280699968338013\\
11 0.695199966430664\\
12 1.88540005683899\\
13 6.16459989547729\\
};
\addlegendentry{Local Flip-Hoisting}
\end{axis}
\end{tikzpicture}
    \vspace{-2em}
  \caption{Inference time cactus plot.}
  \label{fig:examples-time}
\end{subfigure}\\
\qquad \quad

\begin{subfigure}[b]{0.8\columnwidth}
  \centering
\begin{tikzpicture}

\definecolor{color0}{rgb}{0.274509803921569,0.509803921568627,0.705882352941177}
\definecolor{color1}{rgb}{0.603921568627451,0.803921568627451,0.196078431372549}

\begin{axis}[
  height=100px,
legend cell align={left},
legend style={
  fill opacity=0.8,
  draw opacity=1,
  text opacity=1,
  at={(0.03,0.97)},
  anchor=north west,
  draw=white!80!black
},
tick align=outside,
tick pos=left,
x grid style={white!69.0196078431373!black},
xlabel={$n$},
xmajorgrids,
xmin=3.31, xmax=14.09,
xtick style={color=black},
xtick={4.2,5.2,6.2,7.2,8.2,9.2,10.2,11.2,12.2,13.2},
xticklabel style={rotate=315.0,anchor=west},
xticklabels={
  1,
  2,
  3,
  4,
  5,
  6,
  7,
  8,
  9,
  10
},
y grid style={white!69.0196078431373!black},
ylabel={BDD Size},
ymajorgrids,
ymin=0, ymax=246859.2,
ytick style={color=black},
width=\textwidth,
font=\small
]
\draw[draw=none,fill=color0] (axis cs:3.8,0) rectangle (axis cs:4.2,173);
\addlegendimage{ybar,ybar legend,draw=none,fill=color0}
\addlegendentry{Original}

\draw[draw=none,fill=color0] (axis cs:4.8,0) rectangle (axis cs:5.2,410);
\draw[draw=none,fill=color0] (axis cs:5.8,0) rectangle (axis cs:6.2,941);
\draw[draw=none,fill=color0] (axis cs:6.8,0) rectangle (axis cs:7.2,2137);
\draw[draw=none,fill=color0] (axis cs:7.8,0) rectangle (axis cs:8.2,4816);
\draw[draw=none,fill=color0] (axis cs:8.8,0) rectangle (axis cs:9.2,10592);
\draw[draw=none,fill=color0] (axis cs:9.8,0) rectangle (axis cs:10.2,23241);
\draw[draw=none,fill=color0] (axis cs:10.8,0) rectangle (axis cs:11.2,50599);
\draw[draw=none,fill=color0] (axis cs:11.8,0) rectangle (axis cs:12.2,109337);
\draw[draw=none,fill=color0] (axis cs:12.8,0) rectangle (axis cs:13.2,235104);
\draw[draw=none,fill=color1] (axis cs:4.2,0) rectangle (axis cs:4.6,82);
\addlegendimage{ybar,ybar legend,draw=none,fill=color1}
\addlegendentry{Local Flip-Hoisting}

\draw[draw=none,fill=color1] (axis cs:5.2,0) rectangle (axis cs:5.6,158);
\draw[draw=none,fill=color1] (axis cs:6.2,0) rectangle (axis cs:6.6,329);
\draw[draw=none,fill=color1] (axis cs:7.2,0) rectangle (axis cs:7.6,642);
\draw[draw=none,fill=color1] (axis cs:8.2,0) rectangle (axis cs:8.6,1306);
\draw[draw=none,fill=color1] (axis cs:9.2,0) rectangle (axis cs:9.6,2629);
\draw[draw=none,fill=color1] (axis cs:10.2,0) rectangle (axis cs:10.6,5247);
\draw[draw=none,fill=color1] (axis cs:11.2,0) rectangle (axis cs:11.6,10490);
\draw[draw=none,fill=color1] (axis cs:12.2,0) rectangle (axis cs:12.6,20980);
\draw[draw=none,fill=color1] (axis cs:13.2,0) rectangle (axis cs:13.6,42005);
\end{axis}

\end{tikzpicture}
  \vspace{-2em}
  \caption{BDD size}
  \label{fig:examples-size}
\end{subfigure}

\caption{Performance of synthetic probabilistic programs of increasing size $n$.}
\label{fig:best-examples}
\end{figure}
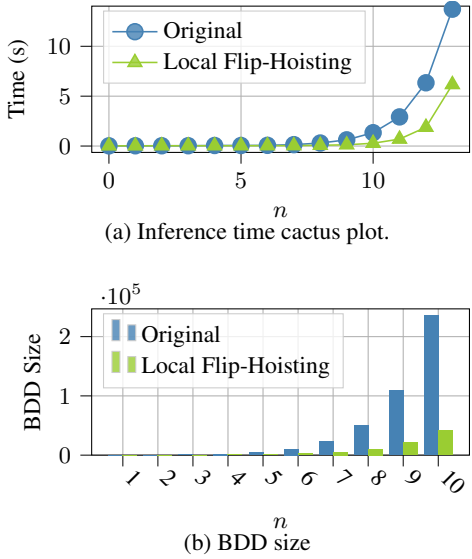

\subsection{Synthetic benchmarks}
\label{sec:synthetic}
Now we seek to answer the question: \emph{what is the best-case scenario for how 
much \fh{} can increase performance in \dice{}}? 
We evaluated local \fh{} on a synthetic benchmark that is designed to have 
many hoisting opportunities and scale in a parameter $n$:
\begin{lstlisting}[language=caml, numbers=none, basicstyle=\ttfamily\scriptsize, mathescape=true]
  let x$_1$ = flip $p_1$ in $\cdots$ let x$_n$ = flip $p_n$ in
  let x$_{n+1}$ = if x$_0$ then if x$_1$ then $\cdots$ if x$_n$ then
    flip $p_{n+1}$ else flip $p_{n+1}'$ $\cdots$ else 
    if x$_1$ then $\cdots$ then flip $p_{n+1}'$ else flip $p_{n+1}$ in
  $\cdots$
  let x$_{2n}$ = if x$_0$ then if x$_1$ then $\cdots$ if x$_n$ then
    flip $p_{2n}$ else flip $p_{2n}'$ $\cdots$ else 
    if x$_1$ then $\cdots$ then flip $p_{2n}'$ else flip $p_{2n}$ in
  x$_{n+1}$ && x$_{n+2}$ && $\cdots$ && x$_{2n}$
\end{lstlisting} 

where $p_i$ and $p_i'$ are randomly generated probabilities and are randomly assigned to \flip{}s
in $x_i$.

Figure~\ref{fig:best-examples} shows a consistent improvement in performance in
both end-to-end inference time and compilation size for the hoisted version on this simple example. 
\fh{} also dramatically reduces the size of the compiled BDD, to less than a
quarter of the original size, experimentally validating Theorem~\ref{thm:size}.
These results show that \fh{} can halve the
time taken to perform inference. Hence, in the best-case scenario, \fh{} 
can significantly improve inference performance for a practical PPL.

\subsection{Evaluation on practical programs}
\label{sec:realistic}
We have shown that \fh{} can improve performance when there are redundant
\flip{}s, and now we answer whether \fh{} is useful in practice. 
There are two questions: (1) whether opportunities for hoisting 
exist in realistic programs, and (2) whether exploiting those hoisting 
opportunities actually speeds up inference. 
We will show affirmative answers to both of these questions for both 
local and global \fh{} on 
a collection of probabilistic programs drawn from probabilistic 
graphical models and probabilistic verification~\citep{HoltzenCAV21}.

\begin{table}
  \small
  \caption{Change in performance for Bayesian network programs after applying local \fh{}. Negative value means an improvement in performance.}
  \medskip
  \label{tab:bn-examples}
  \centering 
  \begin{tabular}{lrrrrrrrrr}
  \toprule
Benchmark & Time & BDD Size\\
\midrule
\textsc{andes} & -61.51\% & -61.25\%\\
\textsc{bn\_78} & -6.22\% & -2.10\%\\
\textsc{bn\_79} & -1.96\% & -3.53\%\\
\textsc{cpcs54} & -6.09\% & -5.27\%\\
\textsc{link} & +0.78\% & 0.00\%\\
\textsc{moissac3} & -32.80\% & -27.73\%\\
\textsc{munin} & -17.08\% & -24.00\%\\
\textsc{munin1} & -44.49\% & -47.07\%\\
\textsc{munin2} & -31.30\% & -32.27\%\\
\textsc{munin3} & -18.57\% & -20.43\%\\
\textsc{munin4} & -22.72\% & -25.30\%\\
\textsc{pathfinder} & -58.98\% & -72.26\%\\
  \bottomrule
  \end{tabular}
\end{table}

\paragraph{Local \fh{}.} First we evaluate local \fh{} on a collection of well-known 
Bayesian network examples  
from the graphical models community, encoded as \dice{} programs.\footnote{See
\url{https://www.bnlearn.com/bnrepository/}.} We especially consider programs that
are challenging for
\dice{}, as defined by having a runtime of greater than 1 second, since
these are the programs sufficiently complex enough to warrant optimizations.
Table~\ref{tab:bn-examples} shows how local \fh{} helps \dice{}'s inference.
Note that for these programs, we measure
the compilation time; that is, the time to build the BDD representation of the program from the \dice{} program. The end-to-end runtime of a particular query is dominated by the compilation due to inference being in linear time to the compiled BDD and polytime proportional to the compile time ~\citep{holtzen2020scaling}. Thus, we report on the compilation time for a more holistic evaluation. 
The results show that hoisting can greatly help inference performance, barring 
the cases where there are no redundant \flip{}s to hoist. Even then,
the incurred cost from running the optimizations on programs that do not
reduce in size is still minor compared to the whole
of the runtime. We can see that
hoisting also provides BDD size benefits on all examples, further
validating Theorem~\ref{thm:size}. In particular, on some examples, such as 
\textsc{pathfinder}, the BDD size is decreased by as much as 72\%.

\begin{table}
  \small
  \caption{Change in performance for probabilistic verification programs after applying local \fh{}. Negative value means an improvement in performance. BDD Size of \textsc{nand} and \textsc{brp} is measured by the iterative \texttt{step} function used in the program.}
  \medskip
  \label{tab:ver-examples}
  \centering 
  \begin{tabular}{lrrrrr}
  \toprule
  Benchmark & Time & BDD Size\\
  \midrule
  \textsc{nand} & -34.93\% & -2.72\%\\
  \textsc{weather} & -24.06\% & -24.35\%\\
  \bottomrule

  \end{tabular}
\end{table}

\paragraph{Hoisting in probabilistic verification.} 
We further evaluate local \fh{} on a completely different set of realistic probabilistic
programs. \citet{HoltzenCAV21} introduced a new class of programs
from the probabilistic verification community;
these models have very different structures from graphical models. We ran
local \fh{} on these models, and found that in 2 out of 7 programs --
\textsc{weather factory} and \textsc{nand} -- there were hoisting opportunities. 
Table~\ref{tab:ver-examples} shows the reduction in compilation runtime 
and in BDD size when local \fh{} is applied.

One program structure present in probabilistic verification and not
Bayesian networks is the use of functions. As \fh{} is an intra-procedural
program analysis, we apply the optimization to each function, and we
observe that the function \textsc{step} in \textsc{nand} benefits from the
optimization. Though the reduction in BDD size of the function is only 2.72\% of
the original, the improvement from \fh{} is compounded substantially because the function is 
called multiple times. These results show that local \fh{} is a general optimization 
that helps probabilistic models other than Bayesian networks.

\paragraph{Hoisting with loopy probabilistic programs.} As seen in the 
probabilistic verification programs, loopy programs 
especially benefit from the effects of \fh{}. \dice{} supports 
statically bounded recursion and thus also bounded loops. Since \fh{} is
an intra-procedural program analysis, the optimization is applied once to
the body of the function (the body of the loop), and the benefits are reaped 
each time the function is called (each iteration of the loop). 
Figure \ref{fig:examples-weather} shows the end-to-end inference time for the 
\textsc{weather} probabilistic verification program with increasing number of loop iterations ($n$). 

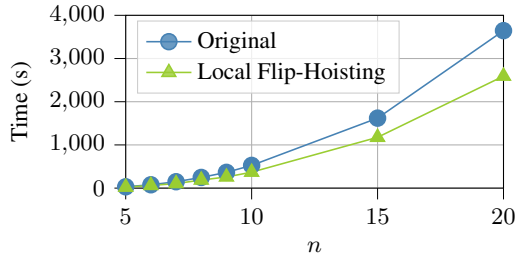
\begin{figure}
\centering
    \begin{subfigure}[t]{0.8\columnwidth}
  \centering
\begin{tikzpicture}

\definecolor{color0}{rgb}{0.274509803921569,0.509803921568627,0.705882352941177}
\definecolor{color1}{rgb}{0.603921568627451,0.803921568627451,0.196078431372549}

\begin{axis}[
  height=110px,
legend cell align={left},
legend style={
  fill opacity=0.8,
  draw opacity=1,
  text opacity=1,
  at={(0.03,0.97)},
  anchor=north west,
  draw=white!80!black
},
tick align=outside,
tick pos=left,
x grid style={white!69.0196078431373!black},
xlabel={$n$},
xmajorgrids,
xmin=5, xmax=20,
xtick style={color=black},
y grid style={white!69.0196078431373!black},
ylabel={Time (s)},
ymajorgrids,
ymin=0, ymax=4000,
ytick style={color=black},
width=\textwidth,
style={font=\small}
]
\addplot [semithick, color0, mark=*, mark size=3, mark options={solid}]
table {%
5 30.7
6 77.48
7 148
8 247
9 365
10 526
15 1620
20 3645
};
\addlegendentry{Original}
\addplot [semithick, color1, mark=triangle*, mark size=3, mark options={solid}]
table {%
5 22
6 62
7 105
8 187
9 258
10 369
15 1176
20 2592
};
\addlegendentry{Local Flip-Hoisting}
\end{axis}

\end{tikzpicture}
    \vspace{-2em}
  \end{subfigure}
  \caption{Inference time of \textsc{weather} with increasing number of loop iterations ($n$).}
  \label{fig:examples-weather}
\end{figure}

\begin{table}
  \small
  \caption{Change in performance for Bayesian network pro-
grams after applying global \fh{} in addition to local \fh{}. Negative value
means an improvement in performance.}
  \medskip
  \label{tab:global-examples}
  \centering 
  \begin{tabular}{lrrrrr}
  \toprule
  Benchmarks & Time & BDD Size \\
  \midrule
  \textsc{emdec6g} & +61.55\% & -0.27\%\\
  \textsc{tcc4e} & +22.16\% & -0.3\%\\
  \textsc{win95pts} & +51.98\% & -0.71\%\\
  \bottomrule
  \end{tabular}
\end{table}

\paragraph{Global Hoisting.} We evaluate the BDD sizes of the 
programs when we applied global \fh{} in addition to local \fh{}.
While it is possible to apply local and global \fh{} in either order -- the difference being which of the \flip{}s are hoisted together in the program and the corresponding LCA that they are hoisted to -- we perform local before global hoisting to minimize tracking data-flow facts.

While more general than local hoisting, global hoisting is a more nuanced
phenomenon that is not as widespread in existing programs. It
still helped some examples in terms of size and their results are presented in
Table~\ref{tab:global-examples}.
The largest improvement from global hoisting was \textsc{emdec6g}, where
the BDD size was decreased by 34 nodes.
The compilation time suffered from the overhead of running the data-flow analysis.
Though marginal, small reductions in BDD size are still important improvements with respect to the \dice{} system. Since \dice{} compilation is modular, the same BDD can be reused in further compilation. Thus, the improvements can compound to have a greater effect, as seen in the loopy programs.
In the future, we aim to broaden
the scope of global hoisting to richer classes of \texttt{if}-expressions, and
further tighten the relationship between global hoisting and the underlying
inference algorithm.

\section{DISCUSSION \& RELATED WORK}
\label{sec:related_work}
Thus far we have demonstrated that \fh{} is a practical and effective program
optimization for speeding up \dice{} inference. Now we will give 
broader context on this contribution, and discuss the extent to 
which \fh{} can possibly be generalized to other kinds of PPLs as well as
where \fh{} sits in the broader landscape of optimizations for PPLs 
and other kinds of probabilistic models.

\paragraph{\fh{} Beyond \dice{}.}
Most of today's PPL optimizations are tied to a single inference 
algorithm. Here we consider the question: is \fh{} a general phenomenon that 
can speed up inference in settings beyond \dice{}'s specific inference 
algorithm?
One natural family of languages to consider are other knowledge-compilation 
based languages that consider other compilation targets for exact inference, 
such as AND/OR search trees, cutset networks, sum-product networks, etc.~\citep{MARINESCU20091457,rahman2014cutset,poon2011sum,mateescu2008and,saad2021sppl}.
\texttt{ProbLog}
is a knowledge compilation based probabilistic logic programming language~\citep{fierens2012inference}. 
Similar to \dice{} it performs weighted model counting on the compiled representation
to compute queries. However, \texttt{ProbLog}
is a logic-based programming language and has a very different semantics from \dice{}. For instance,
\texttt{ProbLog} does not have \texttt{if}-statements, and so a strategy of 
identifying locally redundant variables is not possible. 
Nonetheless, we performed a small-scale preliminary analysis where we 
manually identified and hoisted redundant \flip{}s in \texttt{ProbLog} programs
to see the extent to which an automated \fh{} approach might be profitable here.
The example we considered 
is a \texttt{ProbLog} variant of the program in Figure~\ref{fig:simple-ex}.
We observed that, averaged over 5 runs, manual hoisting yielded a 39\% speedup 
in \texttt{ProbLog}'s inference: a promising preliminary evidence that a variant 
of \fh{} would be profitable for \texttt{ProbLog}.

\paragraph{Probabilistic Graphical Models.} 
Now we discuss the broader context of \fh{} and closely related works.
Inference via knowledge compilation has been studied extensively in the areas of
Bayesian networks and graphical
models, and a number of significant optimizations were developed in that
context~\citep{darwiche2002knowledge,sang2005performing,chavira2008probabilistic, darwiche2009modeling,choi2013compiling, dudek2020addmc, dilkas2021weighted, sanner2005affine, boutilier96context}.
In particular, \citet{chavira2008probabilistic} identified
\emph{parameter sharing} as an important optimization for speeding up exact
probabilistic inference in graphical models by exploiting repeated parameters
within a conditional probability table (CPT) while encoding a graphical model into a logical representation.

There are important differences between parameter sharing and \fh{}.
First, \fh{} has \emph{global scope}. Parameter sharing is limited to exploiting
repeated parameters \emph{within a single CPT}, while global \fh{} is
a whole-program analysis that can hoist repeated parameters \emph{across CPTs}.
More generally, global hoisting is able to identify
\texttt{flip} redundancies in the presence of complex control flow.
Second, \fh{} applies to arbitrary probabilistic programs. When applied to
Bayesian networks encoded as \dice{} programs, 
local hoisting optimizes the repeated parameters that parameter sharing 
would have targeted. However, \fh{} can be applied to program structures
like conditionals, functions, branching, and tuples, and thus the optimization works for 
a much broader set of probabilistic models than just the Bayesian networks that
parameter sharing applies to, such as probabilistic verification programs.
Our setting is more general as a program analysis and applies to broader classes of models. 
Thus, \fh{} can be thought of strictly as a generalization of this approach. 

\paragraph{Slicing and Static Analysis.}
Slicing is a classic and widely-applied technique in traditional optimizing compilers for reducing code size by trimming code that is not required for the program to exhibit its intended behavior.
These techniques have been generalized to probabilistic programs~\citep{hur2014slicing, amtoft2020theory}. 
These methods are orthogonal to ours: programs may contain redundant \flip{}s without any possibility for slicing and vice versa. Common sub-expression elimination attempts to reduce redundant code by computing
the result once, storing it in a local variable, and re-using the local variable many times~\citep{aho1986compilers}. This technique relies on sub-expressions being used
more than once under the same path conditions to improve performance, while \fh{} targets
\flip{}s that have inconsistent path conditions.

\paragraph{Symbolic Optimizations.} Probabilistic programming systems like Psi
and Hakaru internally represent probability distributions symbolically using
computer algebra systems~\citep{gehr2016psi, gehr2020lambdapsi,
narayanan2016probabilistic}. Techniques such as delayed sampling~\citep{murray2018delayed} and semi-symbolic inference~\citep{atkinson2022} also represent probability distributions symbolically, using conjugate priors to solve the system analytically as much as possible, before resorting to sampling. These methods differ from ours as 
they operate on and optimize an internal symbolic representation 
instead of the original program.
A related idea is to use knowledge of conjugacy between distributions
to simplify programs; this is employed by systems like BUGS, JAGS, and Autoconj, but it
exploits a distinct structure from \fh{} where no conjugacy is
necessary~\citep{plummer2003jags, thomas1992bugs, hoffman2018autoconj}. 

\paragraph{Exploiting Structure.} Scaling inference by exploiting structure is
a well-known and widely studied approach. Lifted inference exploits global symmetries 
of the underlying probability distribution in order to scale~\citep{poole2003first, braz2005lifted, kersting2012lifted, ahmadi2013exploiting}; \fh{} exploits local redundant parameters in the distribution,
which may or may not correspond with the kinds of symmetries exploited by lifted inference.
SampleSearch exploits determinism in importance sampling algorithms, whereas our work exploits a broader set of local structure~\citep{gogate2011samplesearch}. 

\paragraph{Probabilistic Program Analyses and Optimizations.} There has been numerous recent
work in techniques for improving inference performance in PPLs in the form of program analyses
and optimizations, such as automatic reparameterization, program
decomposition, and amortization~\citep{gorinova2020automatic, zhou2020divide, ritchie2016deep}.
Similar to how \fh{} analyzes path conditions to find redundant variables, SYMPAIS analyzes path conditions to estimate the satisfaction probability of numerical constraints ~\citep{luo2020sympais}. Another approach to reducing redundant computation using static
analysis is proposed in~\citet{nori2014r2}, where R2 uses static analysis to reject samples early if it is possible to determine that they will violate subsequent observations. This method
does not focus on exact inference and so is not applicable to the same type of programs that
\fh{} targets.

\section{CONCLUSION \& FUTURE WORK}
\label{sec:conclusion}

We present \fh{}, a program analysis and associated optimization for
knowledge compilation based discrete probabilistic programs. \fh{} empirically makes
inference more efficient on a range of programs from the literature on probabilistic graphical models and probabilistic verification.
In the future, we anticipate extending \fh{} to hoist continuous random variables or to optimize approximate inference using a similar approach and pursuing forms of \fh{} across procedures.
Long term, we expect \fh{} to become part of a standard suite of probabilistic
program optimizations, providing efficiency even for programs written by non-expert users in the underlying probabilistic inference algorithm. 

\bibliography{main}

\newpage
\onecolumn

\appendix
\section{ADDITIONAL RESULTS}
We present here the full tables of all experiment results for every benchmark. We imposed a timeout threshold of 20 minutes.

\subsection{Synthetic Programs}
Table~\ref{tab:ex_time} shows the end-to-end inference runtime for the  synthetic 
programs, and Table~\ref{tab:ex_size} shows the BDD size for these programs.

\begin{table}[h]
\small
\centering
\begin{minipage}{0.35\linewidth}
\centering
\caption{Inference runtime for synthetic programs.}
\medskip
\begin{tabular}{rr}
\toprule
Original (s) & Local Flip-Hoisting (s)\\
\midrule
0.04 & 0.04 \\
0.03 & 0.03 \\
0.03 & 0.03 \\
0.03 & 0.03 \\
0.04 & \textbf{0.03} \\
0.05 & \textbf{0.03} \\
0.09 & \textbf{0.04} \\
0.17 & \textbf{0.04} \\
0.37 & \textbf{0.07} \\
0.81 & \textbf{0.17} \\
1.65 & \textbf{0.34} \\
3.56 & \textbf{0.82} \\
7.42 & \textbf{2.37} \\
13.36 & \textbf{6.39} \\
\bottomrule
\label{tab:ex_time}
\end{tabular}
\end{minipage}
\qquad\qquad
\begin{minipage}{0.4\linewidth}
\centering
\caption{Size of compiled BDDs for synthetic programs.}
\medskip
\begin{tabular}{rr}
\toprule
Original & Local Flip-Hoisting \\
\midrule
3 & 3 \\
7 & 7 \\
31 & \textbf{16} \\
72 & \textbf{40} \\
173 & \textbf{82} \\
410 & \textbf{158} \\
941 & \textbf{329} \\
2,137 & \textbf{642} \\
4,816 & \textbf{1,306} \\
10,592 & \textbf{2,629} \\
23,241 & \textbf{5,247} \\
50,599 & \textbf{10,490} \\
109,337 & \textbf{20,980} \\
235,104 & \textbf{42,005} \\
\bottomrule
\label{tab:ex_size}
\end{tabular}
\end{minipage}
\end{table}

\subsection{Bayesian Network Programs}
Table~\ref{tab:bn_time} shows the compilation runtime for the Bayesian network
programs, and Table~\ref{tab:bn_size} shows the BDD size for these programs.
Table~\ref{tab:global_time} shows the compilation runtime for running global \fh{}
on the Bayesian network programs that show improvement in BDD size, and 
Table~\ref{tab:global_size} shows the BDD size for these programs.

\begin{table}[h]
\small
\centering
\begin{minipage}{0.4\linewidth}
\centering
\caption{Compilation runtime for Bayesian network programs.}
\medskip
\begin{tabular}{lrr}
\toprule
Benchmarks & Original (s) & Local Flip-Hoisting (s) \\
\midrule
\textsc{andes} & 7.52 & \textbf{2.90} \\
\textsc{bn\_78} & 4.97 & \textbf{4.66} \\
\textsc{bn\_79} & 26.78 & \textbf{26.25} \\
\textsc{cpcs54} & 1.55 & \textbf{1.46} \\
\textsc{link} & \textbf{168.88} & 170.20 \\
\textsc{moissac3} & 2.13 & \textbf{1.43} \\
\textsc{munin} & 20.72 & \textbf{17.18} \\
\textsc{munin1} & 19.53 & \textbf{10.84} \\
\textsc{munin2} & 34.85 & \textbf{23.94} \\
\textsc{munin3} & 50.73 & \textbf{41.31} \\
\textsc{munin4} & 20.55 & \textbf{15.88} \\
\textsc{pathfinder} & 4.53 & \textbf{1.86} \\
\bottomrule
\end{tabular}
\label{tab:bn_time}
\end{minipage}
\qquad\qquad
\begin{minipage}{0.4\linewidth}
\centering
\caption{Size of compiled BDDs for Bayesian network programs.}
\medskip
\begin{tabular}{lrr}
\toprule
Benchmarks & Original & Local Flip-Hoisting \\
\midrule
\textsc{andes} & 8,520,656 & \textbf{3,301,994} \\
\textsc{bn\_78} & 3,305,000 & \textbf{3,235,616} \\
\textsc{bn\_79} & 21,200,743 & \textbf{20,453,340} \\
\textsc{cpcs54} & 2,438,881 & \textbf{2,310,466} \\
\textsc{link} & 14,964,749 & 14,964,749 \\
\textsc{moissac3} & 498,240 & \textbf{360,070} \\
\textsc{munin} & 4,528,204 & \textbf{3,441,252} \\
\textsc{munin1} & 4,328,101 & \textbf{2,290,665} \\
\textsc{munin2} & 9,285,498 & \textbf{6,289,405} \\
\textsc{munin3} & 13,528,675 & \textbf{10,764,714} \\
\textsc{munin4} & 4,643,895 & \textbf{3,468,929} \\
\textsc{pathfinder} & 61,157 & \textbf{16,968} \\
\bottomrule
\end{tabular}
\label{tab:bn_size}
\end{minipage}
\end{table}

\begin{table}[h]
\small
\centering
\begin{minipage}{0.4\linewidth}
\centering
\caption{Compilation runtime for global \fh{} in addition to local \fh{}.}
\medskip
\begin{tabular}{lrr}
\toprule
  Benchmarks & Local & Global \\
  \midrule
  \textsc{emdec6g} & \textbf{0.0814} & 0.1315 \\
  \textsc{tcc4e} & \textbf{0.088} & 0.1075 \\
  \textsc{win95pts} & \textbf{0.0329} & 0.05 \\
\bottomrule
\end{tabular}
\label{tab:global_time}
\end{minipage}
\qquad\qquad
\begin{minipage}{0.4\linewidth}
\centering
\caption{Size of compiled BDDs for global \fh{} in addition to local \fh{}.}
\medskip
\begin{tabular}{lrr}
\toprule
  Benchmarks & Local & Global \\
  \midrule
  \textsc{emdec6g} & 12,214 & \textbf{12,180} \\
  \textsc{tcc4e} & 1,321 & \textbf{1,317} \\
  \textsc{win95pts} & 982 & \textbf{975} \\
\bottomrule
\end{tabular}
\label{tab:global_size}
\end{minipage}
\end{table}

\subsection{Probabilistic Verification Programs}
Table~\ref{tab:ver_time} shows the compilation runtime for the probabilistic verification
programs, and Table~\ref{tab:ver_size} shows the BDD size for these programs. Some programs show improvements in the sizes of the \textsc{step} function rather than in the state BDD.

\begin{table}[h]
\small
\centering
\begin{minipage}{0.4\linewidth}
\centering
\caption{Compilation runtime for probabilistic verification programs.}
\medskip
\begin{tabular}{lrr}
\toprule
Benchmarks & Original (s) & Local Flip-Hoisting (s) \\
\midrule
\textsc{nand} & 72.98 & \textbf{47.49} \\
\textsc{weather} & 458.60 & \textbf{348.25} \\
\textsc{adv-grid} & \textbf{58.12} & 62.88 \\
\textsc{brp} & 0.74 & \textbf{0.69} \\
\textsc{guymc} & 6.40 & \textbf{6.22} \\
\textsc{queues} & \textbf{242.18} & 326.68 \\
\textsc{motiv} & 0.07 & \textbf{0.03} \\
\bottomrule
\end{tabular}
\label{tab:ver_time}
\end{minipage}
\qquad\qquad
\begin{minipage}{0.4\linewidth}
\centering
\caption{Size of compiled BDDs for probabilistic verification programs.}
\medskip
\begin{tabular}{lrr}
\toprule
Benchmarks & Original & Local Flip-Hoisting \\
\midrule
\textsc{nand (step)} & 6,761 & \textbf{6,577} \\
\textsc{weather} & 350,155 & \textbf{264,909} \\
\textsc{adv-grid} & 302 & 302 \\
\textsc{brp (step)} & 4,255 & \textbf{4,253} \\
\textsc{guymc} & 21,139 & 21,139 \\
\textsc{queues} & 15,812,199 & 15,812,199 \\
\textsc{motiv} & 44 & 44 \\
\bottomrule
\end{tabular}
\label{tab:ver_size}
\end{minipage}
\end{table}

\subsection{Loopy Programs}
Table~\ref{tab:weather_time} shows the inference runtime for the \textsc{weather} probabilistic program as $n$, the number of loop iterations run, increases.

\begin{table}[h]
\small
\centering
\caption{Inference runtime for \textsc{weather} programs.}
\medskip
\begin{tabular}{lrr}
\toprule
$n$ & Original (s) & Local Flip-Hoisting(s) \\
\midrule
\textsc{5} & 30.73 & \textbf{22.38} \\
\textsc{6} & 77.48 & \textbf{62.12} \\
\textsc{7} & 148.26 & \textbf{105.60} \\
\textsc{8} & 247.28 & \textbf{187.97} \\
\textsc{9} & 365.87 & \textbf{258.61} \\
\textsc{10} & 526.65 & \textbf{369.88} \\
\textsc{15} & 1620.38 & \textbf{1176.86} \\
\textsc{20} & 3645.16 & \textbf{2592.16} \\
\bottomrule
\end{tabular}
\label{tab:weather_time}
\end{table}

\clearpage

\subsection{\texttt{ProbLog}}
We generated an extended version of the motivating example program with and without
hoisted parameters in \texttt{ProbLog}. The program is a chain of variables conditioned
on the previous two variables,

\begin{center}
\begin{tabular}{c}
\begin{lstlisting}[language=prolog, numbers=none, basicstyle=\ttfamily\scriptsize, mathescape=true]
    0.1::x$_0$. 0.2::x$_1$. 
    0.3::x$_2$p1. 
    0.4::x$_2$p2. 
    0.4::x$_2$p3. 
    0.3::x$_2$p4.
    x$_2$ :- x$_0$, x$_1$, x$_2$p1.
    x$_2$ :- \+x$_0$, x$_1$, x$_2$p2.
    x$_2$ :- x$_0$, \+x$_1$, x$_2$p3.
    x$_2$ :- \+x$_0$, \+x$_1$, x$_2$p4.
    ...
    0.3::x$_n$p1.
    0.4::x$_n$p2. 
    0.4::x$_n$p3. 
    0.3::x$_n$p4.
    x$_n$ :- x$_{n-2}$, x$_{n-1}$, x$_n$p1.
    x$_n$ :- \+x$_{n-2}$, x$_{n-1}$, x$_n$p2.
    x$_n$ :- x$_{n-2}$, \+x$_{n-1}$, x$_n$p3.
    x$_n$ :- \+x$_{n-2}$, \+x$_{n-1}$, x$_n$p4.
\end{lstlisting}
\end{tabular}
\end{center}

and the hoisted version,

\begin{center}
\begin{tabular}{c}
\begin{lstlisting}[language=prolog, numbers=none, basicstyle=\ttfamily\scriptsize, mathescape=true]
    0.1::x$_0$. 
    0.2::x$_1$. 
    0.3::x$_2$p1. 
    0.4::x$_2$p2. 
    x$_2$ :- x$_0$, x$_1$, x$_2$p1.
    x$_2$ :- \+x$_0$, x$_1$, x$_2$p2.
    x$_2$ :- x$_0$, \+x$_1$, x$_2$p2.
    x$_2$ :- \+x$_0$, \+x$_1$, x$_2$p1.
    ...
    0.3::x$_n$p1. 
    0.4::x$_n$p2. 
    x$_n$ :- x$_{n-2}$, x$_{n-1}$, x$_n$p1.
    x$_n$ :- \+x$_{n-2}$, x$_{n-1}$, x$_n$p2.
    x$_n$ :- x$_{n-2}$, \+x$_{n-1}$, x$_n$p2.
    x$_n$ :- \+x$_{n-2}$, \+x$_{n-1}$, x$_n$p1.
\end{lstlisting}

\end{tabular}
\end{center}

We present in Table~\ref{tab:problog_time} the runtime of each version of the program.

\begin{table}[h]
\small
\centering
\caption{\texttt{ProbLog} inference runtime with and without manual hoisting.}
\medskip
\begin{tabular}{cc}
\toprule
Original (s) & Hoisted (s) \\
\midrule
13.87 & \textbf{8.61} \\
14.13 & \textbf{8.56} \\
14.11 & \textbf{8.68} \\
14.38 & \textbf{8.61} \\
14.13 & \textbf{8.49} \\
\bottomrule
\end{tabular}
\label{tab:problog_time}
\end{table}

\section{\texttt{flip}-HOISTING BAYESIAN NETWORKS}

\fh{} generalizes parameter sharing~\citep{chavira2008probabilistic} to 
probabilistic programs by way of 
reducing redundant random variables. To illustrate how \fh{}
optimizes repeated parameters in the context of Bayesian networks, we present here
an example model.

Figure~\ref{fig:bnencodeA} shows an example discrete Bayesian network on two
variables, $A$ and $B$. The variable $A$ takes on values in the domain $\{0, 1,
2\}$, and $B$ on the domain $\{0, 1\}$. The conditional probability tables (CPTs)
are given in Figure~\ref{fig:bnencodeB}.
We can see that the CPT for $\Pr(B \mid A)$ has repeated parameters: $\Pr(B = 0 \mid A =
1) = \Pr(B = 0 \mid A = 2)$. \citet{chavira2008probabilistic} showed how to
exploit repetitious parameters while encoding a graphical model into a logical
representation. 

Figure~\ref{fig:bnencodeC} represents this 
Bayesian network as a \dice{}
program. It requires the use of the keyword \texttt{discrete}, which defines
a discrete probability distribution over integer values. 
\texttt{discrete} is syntactic sugar that is turned into a series of \flip{}s before compilation.
To define the conditional
distribution $\Pr(B \mid A)$, we branch on each possible value of $A$, flipping
a differently weighted coin for each. Finally,
we return a tuple $(A, B)$ which represents the distribution
on all values the variables in the Bayesian network can jointly take.
\texttt{flip 0.2} is redundant here, since the two copies occur on two different
branches of the \texttt{if}-expression; they would be hoisted when \fh{}
is applied. We can see that these hoisted \flip{}s correspond to the repeated
parameters that would be targeted by parameter sharing. However, \fh{} is a program
analysis and optimization that can be applied to any \dice{} program, not only
Bayesian networks. Thus, \fh{} can benefit a broader set of probabilistic models
and can be thought of strictly as a generalization of parameter sharing.

\begin{figure}
  \centering
\begin{subfigure}[b]{0.45\linewidth}
  \centering
  \begin{tikzpicture}[style={node distance=1.5cm}]
\node[circle, draw] (a) {$A$};
\node[circle, draw, right of = a] (b) {$B$};
\draw[-stealth] (a) -- (b);
  \end{tikzpicture}
  \vspace{1.5em}
  \caption{A simple Bayesian network with 2 variables.}
  \label{fig:bnencodeA}
\end{subfigure}
\qquad
\begin{subfigure}[b]{0.45\linewidth}
  \centering
\begin{lstlisting}[language=caml, escapechar=|, basicstyle=\ttfamily\scriptsize]
let A = 
  discrete(0.2, 0.3, 0.5) 
in
let B = 
  if A==0 then flip 0.1 
  else if A==1 then 
     flip 0.2 
  else flip 0.2 in (A, B)|\label{line:ret}|
\end{lstlisting}
  \caption{A \dice{} encoding of the Bayesian network.}
  \label{fig:bnencodeC}
\end{subfigure}
\bigskip

\begin{subfigure}[b]{\linewidth}
  \small
  \centering
  \begin{tabular}{l|l}
    \toprule
    $A$ & $\Pr(A)$ \\
    \midrule
    0 & 0.2 \\
    1 & 0.3 \\
    2 & 0.5 \\
    \bottomrule
  \end{tabular}
  \qquad\quad
   \begin{tabular}{ll|l}
    \toprule
    $A$ & $B$ & $\Pr(B \mid A)$ \\
    \midrule
     0 & 0 & 0.1 \\
     0 & 1 & 0.9 \\
     1 & 0 & 0.2 \\
     1 & 1 & 0.8 \\
     2 & 0 & 0.2 \\
     2 & 1 & 0.8 \\
    \bottomrule
  \end{tabular}
  \caption{The CPTs for $A$ and $B$.}
  \label{fig:bnencodeB}
\end{subfigure}

\caption{Bayesian network encoding as \dice{} program.}
  \label{fig:bnencode}
\end{figure}


\end{document}